\documentclass[11pt]{article}
\usepackage{geometry}
\usepackage{xspace}
\usepackage{rotating}
\usepackage{bibunits} 
\defaultbibliographystyle{apalike}

\usepackage{amsthm}
\usepackage{amsmath,amssymb,enumerate}

\usepackage{thmtools}
\usepackage{thm-restate}

\declaretheorem[name=Theorem,refname={theorem,theorems},Refname={Theorem,Theorems},numberwithin=section]{theorem}

\declaretheorem[name=Corollary,refname={corollary,corollaries},Refname={Corollary,Corollaries},sibling=theorem]{corollary}

\usepackage[usenames,dvipsnames]{xcolor}
\newcommand{\todoc}[2][]{\todo[color=Apricot!30,#1]{#2}}
\newcommand{\todoar}[2][]{\todo[color=Blue!30,#1]{#2}}

\usepackage{graphicx} 
\usepackage{subfigure} 

\usepackage{natbib}

\usepackage{algorithm}
\usepackage{algorithmic}

\usepackage[bookmarks=false,colorlinks=true,linkcolor=blue,urlcolor=green,citecolor=red,pdfstartview=FitH]{hyperref}
\usepackage{url}

\usepackage{epstopdf}
\usepackage{xcolor}

\def\argmax{\mathop{\rm arg\, max}}
\def\argmin{\mathop{\rm arg\, min}}

\DeclareMathOperator{\tr}{tr}

\def\hg{\hat{g}}

\def\t{\theta}
\def\tit{\tilde{\theta}}

\def\D{\Delta}

\def\F{{\mathcal F}}
\def\H{{\mathcal H}}

\def\L{{\mathcal L}}

\def\S{{\mathcal S}}

\def\X{{\mathcal X}}

\def\d{\delta}

\def\R{\mathbb{R}}
\def\real{\mathbb{R}}

\def\I{\mathcal{I}}

\newcommand{\EEp}[1]{\mathbb{E}\left[#1\right]}

\newcommand{\Prob}[1]{\mathbb{P}\left(#1\right)}
\newcommand{\EEpcond}[2]{\mathbb{E}\left[\left.#1\right|#2\right]}

\def\tr{{\mathop{\rm tr}}}
\def\transpose{{^\top}}
\def\grad{{\nabla}}

\def\d{{\,\,\mbox{d}}}
\def\loss{{L}}

\newcommand{\ip}[1]{\left\langle #1 \right\rangle}
\def\beqa{\begin{eqnarray}}
\def\eeqa{\end{eqnarray}}
\def\beqann{\begin{eqnarray*}}
\def\eeqann{\end{eqnarray*}}

\newcommand{\W}{{\cal W}}
\newcommand{\ra}{\to}
\newcommand{\eqdef}{\doteq} 

\newcommand{\FF}{\mathcal{F}}
\DeclareMathOperator{\Pen}{\text{Pen}}
\newcommand{\cset}[2]{\left\{ #1\,:\,#2 \right\}}

\DeclareMathOperator*{\minimize}{\text{minimize}}
\newcommand{\st}{\text{s.t. }}
\newcommand{\GradSampler}{\text{\tt GradSampler}}
\newcommand{\one}[1]{\mathbb{I}_{\{#1\}}}

\newcommand{\norm}[1]{\left\|#1\right\|}

\newcommand{\kernel}{\kappa}
\newcommand{\Kernel}{{\mathcal K}}
\newcommand{\hd}{\hat{d}}
\newcommand{\MI}{{R_{1:d}}}
\newcommand{\mi}{{r_{1:d}}}

\newcommand{\Stoch}{{\sc Stoch}\xspace}
\newcommand{\Kloft}{{\sc Kloft}\xspace}
\newcommand{\Bach}{{\sc Bach}\xspace}
\newcommand{\Cortes}{{\sc Cortes}\xspace}
\newcommand{\UCD}{{\sc UCD}\xspace}
\newcommand{\Uniform}{{\sc Uniform}\xspace}

\usepackage{graphicx} 
\usepackage{subfigure} 
\usepackage{afterpage}

\usepackage{natbib}

\usepackage{algorithm}
\usepackage{algorithmic}

\usepackage{hyperref}

\usepackage[backgroundcolor = White,textwidth=\marginparwidth,disable]{todonotes}

\title{A Randomized Mirror Descent Algorithm \\for Large Scale Multiple Kernel Learning}
\date{}

\author{
\\ \\Arash Afkanpour \hspace{30mm} Andr\'{a}s Gy\"{o}rgy \\ \\  
\hspace{2mm}Csaba Szepesv\'{a}ri \hspace{30mm} Michael Bowling	\\ \\
\\Department of Computing Science\\
University of Alberta\\
Edmonton, AB\\
\{\texttt{afkanpou,gyorgy,szepesva,mbowling}\}\texttt{@ualberta.ca} \\
}

\begin{document} 

\maketitle


\begin{abstract} 
We consider the problem of simultaneously learning to linearly combine a very large number of kernels and learn a good predictor based on the learnt kernel.
When the number of kernels $d$ to be combined is very large, multiple kernel learning methods whose computational cost scales linearly in $d$ are intractable.
We propose a randomized version of the mirror descent algorithm to overcome this issue,
 under the objective of minimizing the group $p$-norm penalized empirical risk.
The key to achieve the required exponential speed-up 
 is the computationally efficient construction of low-variance estimates of the gradient.
We propose importance sampling based estimates, and
find that the ideal distribution samples a coordinate with a probability proportional to the magnitude of the corresponding gradient.
We show the surprising result that in the case of learning the coefficients of a polynomial kernel,
 the combinatorial structure of the base kernels to be combined allows the implementation of
 sampling from this distribution to run in $O(\log(d))$ time, making 
 the total computational cost of the method  
 to achieve an $\epsilon$-optimal solution
 to be $O(\log(d)/\epsilon^2)$,
 thereby allowing our method to operate for very large values of $d$.
Experiments with simulated and real data confirm that the new algorithm is computationally more efficient than its state-of-the-art alternatives.
\end{abstract}

\section{Introduction}
We look into the computational challenge of finding a good predictor in a multiple kernel learning (MKL) setting where the number of kernels is very large. In particular, we are interested in cases where the base kernels come from a space with combinatorial structure and thus their number $d$ could be exponentially large. 
Just like some previous works
\citep[e.g.][]{rakotomamonjy2008simplemkl,XuJiKiLyu08,NaDiRa09}
we start with the approach that views the MKL problem as a nested, large scale convex optimization problem, where the first layer optimizes the weights of the kernels to be combined. More specifically, as the objective we minimize the group $p$-norm penalized empirical risk.
However, as opposed to these works whose underlying iterative methods have a complexity of  $\Omega(d)$ for just any one iteration, following \citep{Nes10, Nes12, ShTe11,RichTa11} we use a randomized coordinate descent method, which was effectively used in these works to decrease the per iteration complexity to $O(1)$. 
The role of randomization in our method is to use it to build an unbiased estimate of the gradient at the most recent iteration.
The issue  then is how the variance (and so the number of iterations required) scales with $d$.
As opposed to the above mentioned works, in this paper we propose to make the distribution over the updated coordinate \emph{dependent on the history}. 
We will argue that sampling from a distribution that is proportional to the magnitude of the gradient vector is desirable to keep the variance (actually, second moment) low and in fact \todoar{does this long sentence have to be italic?}\emph{we will show that there are interesting cases of MKL (in particular, the case of combining kernels coming from a polynomial family of kernels) when efficient sampling (i.e., sampling at a cost of $O(\log d)$) is feasible from this distribution}. Then,  the variance is  controlled by the a priori weights put on the kernels, making it potentially independent of $d$. \todoar{does this long sentence have to be italic?}\emph{Under these favorable conditions (and in particular, for the polynomial kernel set with some specific prior weights), the complexity of the method as a function of $d$ becomes logarithmic, which makes our MKL algorithm feasible even for large scale problems.}
This is to be contrasted to  the approach of \citet{Nes10,Nes12} where a fixed distribution is used and where the {\em a priori} bounds on the method's convergence rate, and, hence, its computational cost to achieve a prescribed precision, will depend linearly on $d$ (note that we are comparing upper bounds here, so the actual complexity could be smaller).
Our algorithm is based on the mirror descent (or mirror descent) algorithm (similar to the work of \citet{RichTa11} who uses uniform distributions).
It is important to mention that there are algorithms designed to handle the case of infinitely many kernels, for example, the algorithms by \citet{argyriou2005learning, argyriou2006dc, gehler2008infinite}. However, these methods lack convergence rate guarantees, and, for example, the consistency for the method of \citet{gehler2008infinite} works only for ``small'' $d$.
The algorithm of \citet{bach2008exploring}, though practically very efficient, suffers from the same deficiency. 
	\todoc{Double check!}
A very interesting proposal by \citet{cortes2009learning} considers learning to combine a large number of kernels and comes with guarantees, though their algorithm restricts the family of kernels in a specific way.

The rest of the paper is organized as follows. The problem is defined formally in Section~\ref{sec:prelim}. Our new algorithm is presented and analyzed in Section~\ref{sec:main}, while its specialized version for learning polynomial kernels is given in Section~\ref{sec:examples}.
Finally, experiments are provided in Section~\ref{sec:experiments}.

\section{Preliminaries} 
\label{sec:prelim}

In this section we give the formal definition of our problem.
Let $\I$ denote a finite index set, indexing the predictors (features) to be combined, and
define the set of predictors considered over the input space $\X$ as
$ 
\FF = \cset{ f_w: \X \ra \real }{ f_w(x) = \sum_{i\in \I} \ip{ w_i, \phi_i(x) }, \quad x\in \X }
$. 
Here
 $\W_i$ is a Hilbert space over the reals, 
 $\phi_i:\X \ra \W_i$ is a feature-map,
 $\ip{x,y}$ is the inner product over the Hilbert space that $x,y$ belong to
 and $w = (w_i)_{i\in \I}\in \W \eqdef \times_{i\in \I} \W_i$ (as an example, $\W_i$ may just be a finite dimensional Euclidean space).
The problem we consider is to solve the optimization problem
\begin{equation}
\label{eq:penrisk}
\minimize \,\, \loss_n(f_w) + \Pen(f_w) \quad \text{subject to } w\in \W\,,
\end{equation}
where  $\Pen(f_w)$ is a penalty that will be specified later, and 
$
\loss_n(f_w)  =\frac1n \sum_{t=1}^n \ell_t( f_w(x_t) )
$
is the empirical risk of predictor $f_w$, defined in terms of the convex losses
 $\ell_t: \real\ra \real$  ($1\le t\le n$) and inputs $x_t\in \X$ ($1\le t\le n$). The solution $w^*$ of the above penalized empirical risk minimization problem is known to have favorable generalization properties under various conditions, see, e.g., \citet{HaTiFri09}.
In supervised learning problems $\ell_t(y) = \ell(y_t,y)$ for some loss function $\ell:\R\times \R \ra \R$, such as 
 the squared-loss, $\ell(y_t,y) =\frac12 (y-y_t)^2$, or the hinge-loss, $\ell_t(y_t,y) = \max(1-y y_t,0)$, where in the former case $y_t\in \real$, while in the latter case $y_t\in \{-1,+1\}$.
We note in passing that  for the sake of simplicity, we shall sometimes abuse notation and write $\loss_n(w)$ for $\loss_n(f_w)$ and even drop the index $n$ when the sample-size is unimportant. \todoc{I decided to write this in terms of weights because $\Pen(f)$ might not be well-defined in our case below. Mention unsupervised learning later.}

As mentioned above, in this paper we consider the special case in~\eqref{eq:penrisk} when the penalty is a so-called group $p$-norm penalty with $1\le p \le 2$, a case considered earlier, e.g., by \citet{kloft2011lp}. Thus our goal is to solve 
\begin{equation}
\label{eq:min1}
\minimize_{w\in \W} \,\, \loss_n(w) + \frac{1}{2} \left(\sum_{i\in \I} \rho_i^p \| w_i\|_2^p\right)^{\frac{2}{p}}\,,
\end{equation}
where the scaling factors $\rho_i>0, i\in\I$, are assumed to be given. 
We introduce the notation $u = (u_i)\in \R^\I$ to denote the column vector obtained from the values $u_i$. \todoc{Why here?} \todoar{This sentence must be rewritten.}

The rationale of using the squared weighted $p$-norm is that for $1\le p < 2$ it is expected to encourage sparsity at the group level which should allow one to handle cases when $\I$ is very large (and the case $p=2$ comes for free from the same analysis). \todoc{Ref?}
The actual form, however, is also chosen for reasons of computational convenience. 
In fact, the reason to use the $2$-norm of the weights is to allow the algorithm to work even with infinite-dimensional feature vectors (and thus weights) by resorting to the kernel trick. To see how this works, just notice that the  penalty in~\eqref{eq:min1} can also be written 
as
\[
\left(\sum_{i\in \I} \rho_i^p \| w_i\|_2^p\right)^{\frac{2}{p}}
=
\inf \left\{ \sum_{i\in\I} \frac{\rho_i^2\|w_i\|_2^2}{\t_i} \,:\, \t\in\Delta_{\frac{p}{2-p}} \right\}\,,
\]
where for $\nu\ge 1$,
$ 
\D_\nu=\{\t\in [0,1]^{|\I|}: \|\t\|_\nu \le 1\}
$ 
is the positive quadrant of the $|\I|$-dimensional $\ell^\nu$-ball \citep[see, e.g.,][Lemma 26]{MiPo05}.
Hence, defining
\[
J(w,\t)=L(w)+ \frac{1}{2} \sum_{i\in\I} \frac{\rho_i^2\|w_i\|_2^2}{\t_i}
\]
for any $w\in\W, \t \in [0,1]^{|\I|}$,
an equivalent form of~\eqref{eq:min1} is 
\begin{equation}
\label{eq:jmin}
\minimize_{w\in \W,\t\in \Delta_{\nu}} \,\, J(w,\t) \\ 
\end{equation}
where 
$\nu = p/(2-p)\in [1,\infty)$ and
we define $0/0=0$ and $u/0=\infty$ for $u>0$, which implies that $w_i=0$ if $\t_i=0$. 
That this minimization problem is indeed equivalent to our original task \eqref{eq:min1} for the chosen value of $\nu$ follows 
from the fact that $J(w,\t)$ is jointly convex in $(w,\t)$.\footnote{Here and in what follows by equivalence we mean that the set of optimums in terms of $w$ (the primary optimization variable) is the same in the two problems.}

Let $\kernel_i:\X\times \X \to\real$ be the reproducing kernel underlying $\phi_i$: $\kernel_i(x,x') = \ip{ \phi_i(x),\phi_i(x') }$ ($x,x'\in \X$) and let $\H_i = H_{\kernel_i}$ the corresponding reproducing kernel Hilbert space (RKHS). 
Then, for any given fixed value of $\t$, the above problem becomes an instance of a standard penalized learning problem in the RKHS $\H_\t$ underlying the kernel $\kernel_\t =  \sum_{i\in\I}  \theta_i \rho_i^{-2} \kernel_i$.
In particular, by the theorem on page 353 in \citet{Aro50}, the problem of finding $w\in \W$ for fixed $\t$ can be seen to be equivalent to \todoc{For finite $\I$..}
$ 
\minimize_{f\in \H_\t} \,\, L(f) + \frac{1}{2} \|f\|_{\H_\t}^2,
$ 
and thus~\eqref{eq:min1} is seen to be equivalent to 
$ 
\minimize_{f\in \H_\t,\t \in \Delta_{\nu}} \,\, L(f) + \frac{1}{2} \|f\|_{\H_\t}^2\,.
$ 
Thus, we see that the method can be thought of as finding the weights of a kernel $\kernel_\t$ and a predictor minimizing the $\H_\t$-norm penalized empirical risk.
This shows that our problem is an instance of multiple kernel learning
 (for an exhaustive survey of MKL, see, e.g., \citealp{gonen2011multiple} and the references therein). 

\section{The new approach}
\label{sec:main}

When $\I$ is small, or moderate in size, the joint-convexity of $J$ allows one to use off-the-shelf solvers to find the joint minimum of $J$. However, when $\I$ is large, off-the-shelf solvers might be slow or they may run out of memory.
Targeting this situation 
 we propose the following approach: Exploiting again that $J(w,\t)$ is jointly convex in $(w,\t)$,  find the optimal weights by
finding the minimizer of 
\[
J(\t) \eqdef \inf_{w} J(w,\t),
\]
or, alternatively, $J(\t) = J(w^*(\t),\t)$, where $w^*(\t) \doteq \arg\min_w J(w,\t)$ (here we have slightly abused notation by reusing the symbol $J$). Note that $J(\t)$ is convex by the joint convexity of $J(w,\t)$.
Also, note that $w^*(\t)$ exists and is well-defined as the minimizer of $J(\cdot,\t)$ is unique for any $\t\in \Delta_{\nu}$ (see also Proposition~\ref{prop:wstar} below).
Again, exploiting the joint convexity of $J(w,\t)$, we find that if $\t^*$ is the minimizer of $J(\t)$, then $w^*(\t^*)$ will be an optimal solution to the original problem~\eqref{eq:min1}. 
To optimize $J(\t)$ we propose to use stochastic gradient descent with artificially injected randomness to avoid the need to fully evaluate the gradient of $J$.
More precisely, our proposed algorithm is an instance of a randomized version of the {\em mirror descent algorithm} \citep{Roc76,Mar78,NY83},
where in each time step only one coordinate of the gradient is sampled.

\subsection{A randomized mirror descent algorithm}

Before giving the algorithm, we need a few definitions. Let $d = |\I|$, $A \subset \R^d$ be nonempty with a convex interior $A^\circ$. We call the function $\Psi: A \to \R$ a \emph{Legendre} (or barrier) \emph{potential}  if it is strictly convex, its partial derivatives exist and are continuous, and for every sequence $\{x_k\} \subset A$ approaching the boundary of $A$, $\lim_{k\to\infty}\|\nabla \Psi(x_k)\|=\infty$.
Here $\nabla$ is the gradient operator: 
$\nabla \Psi(x) = (\frac{\partial}{\partial x} \Psi(x))^\top$ is the gradient of  $\Psi$. When $\nabla$ is applied to a non-smooth convex function $J'(\t)$ ($J$ may be such without additional assumptions) then $\nabla J'(\t)$ is defined as any subgradient of $J'$ at $\t$. 
\todoc{Need to choose a norm when $\I$ is infinite in the definition of Legendre functions.}
The corresponding Bregman-divergence $D_\Psi:A\times A^\circ \to\R$ is defined as
$ 
D_\Psi(\t,\t')=\Psi(\t)-\Psi(\t')-\langle \nabla\Psi(\t'),\t-\t' \rangle
$. 
The Bregman projection  $\Pi_{\Psi,K}: A^\circ\to K$ corresponding to the Legendre potential $\Psi$ and  a closed convex set $K \subset \R^d$ such that $K \cap A \neq \emptyset$ is defined, for all $\t \in A^\circ$ as 
$ 
\Pi_{\Psi,K}(\t)=\argmin_{\t' \in K \cap A} D_\psi(\t',\t)
$. 

Algorithm~\ref{alg:proximal} shows a randomized version of the standard mirror descent method with an unbiased gradient estimate.  By assumption, $\eta_{k}>0$ is deterministic. Note that step~\ref{algstep:proj} of the algorithm is well-defined since $\tit^{(k)} \in A^\circ$ by the assumption that $\|\nabla \Psi(x)\|$ tends to infinity as $x$ approaches the boundary of $A$.
%
\begin{algorithm}[tb]
	\caption{Randomized mirror descent algorithm}
	\label{alg:proximal}
\begin{algorithmic}[1]
	\STATE \textbf{Input:} 
	 	$A,K \subset \R^d$, where $K$ is closed and convex with $K \cap A \neq \emptyset$, 
	 	$\Psi: A \to \R$ Legendre, step sizes $\{\eta_k\}$, a subroutine, $\GradSampler$, to sample the gradient of $J$ at an arbitrary vector $\t\ge 0$
	\STATE \textbf{Initialization:} $\t^{(0)}=\argmin_{\t \in K\cap A} \Psi(\t)$, $k = 0$.
	\REPEAT
		\STATE $k=k+1$.
		\STATE Obtain $\hg_k = \GradSampler(\t^{(k-1)})$
               \label{algstep:grad}
		\STATE $\tit^{(k)} = \argmin_{\t \in A} \left\{ \eta_{k-1} \langle \hg_k, \t \rangle + D_{\Psi}(\t,\t^{(k-1)})\right\}$.
                \label{algstep:thetamin}
		\STATE $\t^{(k)} = \Pi_{\Psi,K}(\tit^{(k)})$.
                \label{algstep:proj}
	\UNTIL{convergence.}
\end{algorithmic}
\end{algorithm}
%
The performance of Algorithm~\ref{alg:proximal} is bounded in the next theorem. The analysis follows the  standard proof technique of analyzing the mirror descent algorithm (see, e.g., \citealp{beck2003mirror}), however, in a slightly more general form than what we have found in the literature. In particular, compared to \citep{NeJuLaSh09,Nes10, Nes12, ShTe11,RichTa11}, our analysis allows for the conditional distribution of the noise in the gradient estimate to be history dependent.
The proof is included in Section~\ref{sec:proofs} in the appendix.
\begin{restatable}{theorem}{theoremPerfBound}
\label{thm:proximal}
Assume that $\Psi$ is $\alpha$-strongly convex  with respect to some norm $\|\cdot\|$ (with dual norm $\|\cdot\|_*$) for some $\alpha>0$, that is, for any $\t \in A^\circ, \t' \in A$
\begin{equation}
\label{eq:psi-strong}
\Psi(\t')-\Psi(\t) \ge \ip{ \nabla \Psi(\t),\t'-\t } + \tfrac\alpha{2}  \|\t'-\t\|^2.
\end{equation}
Suppose, furthermore, that Algorithm~\ref{alg:proximal} is run for $T$ time steps. 
For $0\le k \le T-1$ let $\F_k$ denote the $\sigma$-algebra generated by $\t_1,\ldots,\t_k$.
Assume that,  for all $1\le k \le T$,  $\hg_k \in \R^d$ is an unbiased estimate of $\nabla J(\t^{(k-1)})$ given $\F_{k-1}$, that is,
\begin{equation}
\label{eq:gradest}
\EEpcond{\hg_k}{\F_{k-1}}= \nabla J(\t^{(k-1)}).
\end{equation}
Further, assume that there exists a deterministic constant $B\ge 0$ such that 
for all $1\le k \le T$,
\begin{equation}
\label{eq:gbound}
\EEpcond{\|\hg_k\|_*^2}{\F_{k-1}} \le B \quad \text{a.s.}
\end{equation}
Finally, assume that $\delta=\sup_{\t' \in K \cap A} \Psi(\t') - \Psi(\t^{(0)})$ is finite.
Then, if $\eta_{k-1}=\sqrt{\frac{2\alpha\delta}{B T}}$ for all $k \ge 1$,  it holds that
\begin{equation}
\label{eq:thm-expectation}
\EEp{J\left(\frac{1}{T}\sum_{k=1}^T \t^{(k-1)}\right)} - \inf_{\t \in K\cap A} J(\t)
\le \sqrt{\frac{2B\delta}{\alpha T}}.
\end{equation}
Furthermore, if
\begin{equation} 
\label{eq:hgkboundedvar}
\|\hg_k\|_*^2 \le B' \quad \text{a.s.}
\end{equation} 
for some deterministic constant $B'$ and $\eta_{k-1}=\sqrt{\frac{2\alpha\delta}{B' T}}$ for all $k \ge 1$ then, for any $0<\epsilon<1$, it holds with probability at least $1-\epsilon$ that
\begin{equation}
\label{eq:thm-highprob}
J\left(\frac{1}{T}\sum_{k=1}^T \t^{(k-1)}\right) - \inf_{\t \in K\cap A} J(\t)
\le \sqrt{\frac{2B'\delta}{\alpha T}} + 4\sqrt{\frac{B'\delta\log\frac{1}{\epsilon}}{\alpha T}}.
\end{equation}
\end{restatable}
The convergence rate in the above theorem can be improved if stronger assumptions are made on $J$, for example if $J$ is assumed to be strongly convex, see, for example, \citep{HaAgKa07, HaKa11}. 

Efficient implementation of Algorithm~\ref{alg:proximal} depends on efficient implementations of steps \ref{algstep:grad}-\ref{algstep:proj}, namely, computing an estimate of the gradient, solving the minimization for $\tit^{(k)}$, and projecting it into $K$. The first problem is related to the choice of gradient estimate we use, which, in turn, depends on the structure of the feature space, while the last two problems depend on the choice of the Legendre function. In the next subsections we examine how these choices can be made to get a practical variant of the algorithm.

\subsection{Application to multiple kernel learning} 
It remains to define the gradient estimates $\hg_k$ in Algorithm~\ref{alg:proximal}. We start by considering importance sampling based estimates.
First, however, let us first verify whether the gradient exist. Along the way, we will also derive some explicit expressions which will help us later.

\paragraph{Closed-form expressions for the gradient.}
Let us first consider how $w^*(\t)$ can be calculated for a fixed value of $\t$. As it will turn out, this calculation will be useful not only when the procedure is stopped (to construct the predictor $f_{w^*(\t)}$ but also during the iterations when we will need to calculate the derivative of $J$ with respect to $\theta_i$. The following proposition summarizes how $w^*(\t)$ can be obtained. Note that this type of result is standard \citep[see, e.g.,][]{shawe2004kernel,scholkopf2002learning}, \todoc{Double check refs. Give page numbers!?} thus we include it only for the sake of completeness (the proof is included in Section~\ref{sec:proofs} in the appendix).
\begin{restatable}{proposition}{propwstar}
\label{prop:wstar}
For $1\le t\le n$, let $\ell_t^*:\real\ra\real$
denote the convex conjugate of $\ell_t$: $\ell_t^*(v) = \sup_{\tau\in\real}\left\{ v\tau -\ell_t(\tau)\right\}$, $v\in \real$.
For $i\in \I$, recall that $\kernel_i(x,x') = \ip{\phi_i(x),\phi_i(x')}$, and let $\Kernel_i = (\kernel_i(x_t,x_s))_{1\le t,s\le n}$ be the $n\times n$ kernel matrix underlying $\kernel_i$ and let $\Kernel_\theta = \sum_{i\in\I} \frac{\t_i}{\rho_i^2} \Kernel_i $ be the kernel matrix underlying $\kernel_\theta = \sum_{i\in\I} \frac{\t_i}{\rho_i^2} \kernel_i$.
Then, for any fixed $\t$, the minimizer $w^*(\t)$ of $J(\cdot,\t)$ satisfies \todoc{Reference? I imagine this proposition was known beforehand..}
\begin{equation}
	w_i^*(\t) = \frac{\t_i}{\rho_i^2} \sum_{t=1}^n \alpha_t^*(\t) \phi_i(x_t), \quad i\in\I\,,
\label{eq:optimal_w}
\end{equation}
where
\begin{equation}
	\alpha^*(\t) = \argmin_{\alpha\in\R^n} \left\{  \frac{1}{2} \alpha\transpose \Kernel_\t \alpha + \frac{1}{n} \sum_{t=1}^n \ell_t^*(-n\alpha_t)\right\}\,.
\label{eq:optimal_alpha}
\end{equation}
\end{restatable}
Based on this proposition, we can compute the predictor $f_{w^*(\t)}$ using the kernels $\{\kernel_i\}_{i\in \I}$ and the dual variables $(\alpha_t^*(\t))_{1\le t \le n}$:
$
f_{w^*(\t)}(x) = \sum_{i\in \I} \ip{ w_i^*(\t), \phi_i(x) } = 
\sum_{t=1}^n \alpha_t^*(\t) \kernel_\theta (x_t,x)\,.
$

Let us now consider the differentiability of $J = J(\t)$ and how to compute its derivatives.
Under proper conditions with standard calculations  \citep[e.g.,][]{rakotomamonjy2008simplemkl} we find that $J$ is differentiable over $\Delta$ and its derivative can be written as%
\footnote{For completeness, the calculations are given in Section~\ref{sec:derivative} in the appendix.}
\begin{equation}
\label{eq:diff1}
\frac{\partial}{\partial \t} J(\t) = -\left(\frac{\alpha^*(\t)\transpose \Kernel_i \alpha^*(\t)}{ \rho_i^2}\right)_{i\in\I}\,.
\end{equation}

\paragraph{Importance sampling based estimates.}
\label{sec:impsampest}
Let $d=|\I|$ and let $e_i, \, i \in \I$ denote the $i^{\rm th}$ unit vector of the standard basis of $\R^{d}$, that is, the $i^{\rm th}$ coordinate of $e_i$ is $1$ while the others are $0$. Introduce
\begin{align}\label{eq:gradki}
g_{k,i} = \ip{ \nabla J(\t^{(k-1)}), e_{i} }, \quad i\in \I
\end{align}
to denote the $i^{\rm th}$ component of the gradient of $J$ in iteration $k$ (that is, $g_{k,i}$ can be computed based on~\eqref{eq:diff1}).
Let $s_{k-1} \in [0,1]^{\I}$ be a distribution over $\I$, computed in some way  based on the information available up to the end of iteration $k-1$ of the algorithm (formally, $s_{k-1}$ is $\F_{k-1}$-measurable). \todoc{For $\I$ uncountable, it has to be a measurable space for this to make sense.}
Define the importance sampling based gradient estimate to be
\beqa
\label{eq:hg_impsamp}
	\hg_{k,i}=\frac{\one{I_k=i}}{s_{k-1,I_k}} g_{k,I_k}, \quad i\in \I,
	\,\,\, \text{where } I_k \sim s_{k-1,\cdot}\,.
\eeqa
That is, the gradient estimate is obtained by first sampling an index from $s_{k-1,\cdot}$ and then
 setting the gradient estimate to be zero at all indices $i\in \I$ except when $i=I_{k}$ in which case its value is set to be the ratio $\frac{ g_{k,I_k}}{s_{k-1,I_k}}$.
It is easy to see that as long as $s_{k-1,i}>0$ holds whenever $ g_{k,i}\not=0$, then it holds that
$ 
	\EEpcond{\hg_k}{\F_{k-1}} = \grad J(\theta^{(k-1)})
$ a.s. 

Let us now derive the conditions under which the second moment of the gradient estimate stays bounded.
Define 
$ 
C_{k-1} = \left \| \nabla J(\t^{(k-1)}) \right\|_1
$. 
Given the expression for the gradient of $J$ shown in~\eqref{eq:diff1}, we see that
 $\sup_{k\ge 1} C_{k-1}<\infty$ will always hold provided that $\alpha^*(\t)$ is continuous since $(\t^{(k-1)})_{k\ge 1}$ is guaranteed to belong to a compact set (the continuity of $\alpha^*$ is discussed in Section~\ref{sec:derivative} in the appendix).

Define the probability distribution $q_{k-1,\cdot}$ as follows:
$
q_{k-1,i}=\frac{1}{C_{k-1}} \left|   g_{k,i}\right|\,,\quad i\in \I
$. 
Then it holds that
$ 
\|\hg_k\|_*^2  
 = \frac{1}{s_{k-1,I_k}^2}  g_{k,I_k}^2 \norm{e_{I_k}}_*^2
 = \frac{q_{k-1,I_k}^2}{s_{k-1,I_k}^2} \, C_{k-1}^2 \norm{e_{I_k}}_*^2
$. 
Therefore, it also holds that
$ 
	\EEpcond{\|\hg_k\|_*^2}{\F_{k-1}} = C_{k-1}^2 \sum_{i\in\I} \frac{q_{k-1,i}^2}{s_{k-1,i}} \|e_i\|_*^2 \leq C_{k-1}^2 \max_{i \in \I} \frac{q_{k-1,i}}{s_{k-1,i}} \|e_i\|_*^2
$. 
This shows that  $\sup_{k\ge 1} \EEpcond{\|\hg_k\|_*^2}{\F_{k-1}} < \infty$ will hold as long as
 $\sup_{k\ge 1}  \max_{i \in \I} \frac{q_{k-1,i}}{s_{k-1,i}}  < \infty$ and $\sup_{k\ge 1} C_{k-1}<\infty$.
 Note that when $s_{k-1} = q_{k-1}$, the gradient estimate becomes $\hg_{k,i} = C_{k-1} \one{I_t=i}$. That is, in this case we see that in order to be able to calculate $\hg_{k,i}$, we need to be able to calculate $C_{k-1}$ efficiently.
 \todoc{Add a proposition that sums up the derivations}

\paragraph{Choosing the potential $\Psi$.}
The efficient sampling of the gradient is not the only practical issue, since the choice of the Legendre function and the convex set $K$ may also cause some complications. For example, if $\Psi(x)=\sum_{i\in\I} x_i (\ln x_i -1)$, then the resulting algorithm is exponential weighting, and one needs to store and update $|\I|$ weights, which is clearly infeasible if $|\I|$ is very large (or infinite). On the other hand, if $\Psi(x)=\tfrac12\|x\|_2^2$ and we project to $K=\Delta_2$, the positive quadrant of the $\ell^2$-ball (with  $A=[0,\infty)^\I$), we obtain a stochastic projected gradient method, shown in Algorithm~\ref{alg:grad}. This is in fact the algorithm that we use in the experiments. Note that in~\eqref{eq:min1} this corresponds to using $p=4/3$. The reason we made this choice is because in this case projection is a simple scaling operation. Had we chosen $K=\Delta_1$, the $\ell^2$-projection would very often cancel many of the nonzero components, resulting in an overall slow progress.
Based on the above calculations and Theorem~\ref{thm:proximal} we obtain the following performance bound for our algorithm. 

\begin{corollary}
Assume that $\alpha^*(\t)$ is continuous on $\Delta_2$. Then
there exists a $C>0$ such that $\|\frac{\partial}{\partial \t} J(\t)\|_1 \le C$ for all $\t \in \D_2$. 
Let $B=\tfrac12 C^2 \max_{i \in \I,1\le k \le T} \frac{q_{k-1,i}}{s_{k-1,i}}$.
If Algorithm~\ref{alg:grad} is run for $T$ steps with
$\eta_{k-1}=\eta=1/\sqrt{BT}, k=1,\ldots,T$, then, for all $\t\in \D_2$,
\[
\EEp{J\left(\frac{1}{T}\sum_{k=1}^T \t^{(k-1)} \right)} - J(\t)
\le \sqrt{\frac{B}{T}}.
\]
\end{corollary}

Note that to implement Algorithm~\ref{alg:grad} efficiently, one has to be able to sample from $s_{k-1,\cdot}$ and compute the importance sampling ratio $g_{k,i}/s_{k,i}$ efficiently for any $k$ and $i$.

\begin{algorithm}[tb]
	\caption{Projected stochastic gradient algorithm.}
	\label{alg:grad}
\begin{algorithmic}[1]
	\STATE \textbf{Initialization:} $\Psi(x)=\tfrac12\|x\|_2^2$,
        $\t^{(0)}_i=0$ for all $i\in\I$, $k=0$, step sizes $\{\eta_k\}$.
	\REPEAT 
		\STATE $k = k + 1$.
		\STATE Sample a gradient estimate $\hg_{k}$ of $g(\t^{(k-1})$ randomly according to \eqref{eq:hg_impsamp}.
		\STATE $\t^{(k)}=\Pi_{\Psi,\D_2}(\t^{(k-1)}-\eta_{k-1} \hg_{k})$.
	\UNTIL{convergence.}
\end{algorithmic}
\end{algorithm}

\todo{$K$ in algorithm \ref{alg:grad} has not been used inside the body of the algorithm. perhaps in the last line?}

\section{Example: Learning polynomial kernels}
\label{sec:examples}

In this section we show how our method can be applied in the context of multiple kernel learning.
We provide an example when the kernels in $\I$ are tensor products of a set of base kernels (this we shall call learning polynomial kernels).
The importance of this example follows from the observation of \citet{gonen2011multiple} that the non-linear kernel learning methods
of \citet{cortes2009learning}, which can be viewed as a restricted form of learning polynomial kernels, are far the best MKL methods in practice and can significantly outperform state-of-the-art SVM with a single kernel or with the uniform combination of kernels.

\label{sec:poly}
Assume that we are given a set of base kernels $\{\kernel_1,\ldots,\kernel_r\}$.
In this section we consider the set $K_D$ of product kernels of  degree at most $D$:
Choose 
$ 
\I = \cset{ (r_1, \ldots, r_d) }{ 0 \le d \le D, 1 \le r_i \le r }
$ 
and the multi-index $\mi = ({r_1}, \ldots, {r_d})\in \I$ defines the kernel $\kernel_{\mi}(x,x') = \prod_{i=1}^d \kernel_{r_i}(x,x')$. For $d=0$ we define $\kernel_{r_{1:0}}(x,x')=1$. Note that indices that are the permutations of each other define the same kernel. On the language of statistical modeling, $\kernel_{\mi}$ models interactions of order $d$ between the features underlying the base kernels $\kernel_1,\ldots,\kernel_r$.
Also note that $|\I| = \Theta( r^{D} )$, that is, the cardinality of $\I$ grows exponentially fast in $D$.

We assume that $\rho_{\mi}$ depends only on $d$, the order of interactions in $\kernel_{\mi}$.
By abusing notation, we will write $\rho_d$ in the rest of this section to emphasize this.%
\footnote{Using importance sampling, more general weights \todoc[inline]{Weights? Was this terminology introduced?}
 can also be accommodated, too without effecting the results as long as the range of weights $(\rho_{\mi})$ is kept under control for all $d$.}
Our proposed algorithm to sample from  $q_{k-1,\cdot}$ is shown in Algorithm~\ref{alg:poly_kernel_sampling}. 
The algorithm is written to return a multi-index $(z_1,\ldots,z_d)$ that is drawn from $q_{k-1,\cdot}$.
The key idea underlying the algorithm is to exploit that $(\sum_{j=1}^r \kernel_j)^d=\sum_{\mi \in \I} \kernel_\mi$.
The correctness of the algorithm is shown in Section~\ref{sec:polysampling}.
In the description of the algorithm $\odot$ denotes the matrix entrywise product (a.k.a.\ Schur, or Hadamard product)
and $A^{\odot s}$ denotes $\underbrace{A\odot \ldots \odot A}_{s}$, and we set the priority of $\odot$  to be higher than that of the ordinary matrix product
(by definition, all the entries of $A^{\odot 0}$ are $1$).

Let us now discuss the complexity of  Algorithm~\ref{alg:poly_kernel_sampling}. For this, first note that computing all the Hadamard products $S^{\odot d'}, d'=0,\ldots,D$ requires $O(Dn^2)$ computations.
Multiplication with $M_{k-1}$ can be done in $O(n^2)$ steps. Finally, note that each iteration of the for loop takes $O(rn^2)$ steps, which results in the overall worst-case complexity of $O(r n^2D)$ if $\alpha^*(\theta_{k-1})$ is readily available. The computational complexity of determining $\alpha^*(\theta_{k-1})$ depends on the exact form of $\ell_t$, and can be done efficiently in many situations: if, for example, $\ell_t$ is the squared loss, then $\alpha^*$ can be computed in $O(n^3)$ time. \todo{I removed hinge loss, it can come back whenever the exact computational complexity is given.}
An obvious improvement to the approach described here, however, would be to subsample the empirical loss $L_n$, which can bring further computational improvements. However, the exploration of this is left for future work.

Finally, note that despite the exponential cardinality of $|\I|$, due to the strong algebraic structure of the space of kernels, $C_{k-1}$ can be calculated efficiently. In fact, it is not hard to see that with the notation of the algorithm, $C_{k-1}= \sum_{d'=0}^D \delta(d')$. This also shows that if $\rho_d$ decays ``fast enough'', $C_{k-1}$ can be bounded independently of the cardinality of $\I$.

\newcommand{\LET}{\leftarrow}
\begin{algorithm}[tb]
\caption{Polynomial kernel sampling. The symbol $\odot$ denotes the Hadamard product/power.}
\label{alg:poly_kernel_sampling}
\begin{algorithmic}[1]
    \STATE \textbf{Input:} $\alpha \in \mathbb{R}^n$, the solution to the dual problem; kernel matrices $\{\Kernel_1,\ldots,\Kernel_r\}$;
    the degree $D$ of the polynomial kernel, the weights $(\rho_0^2,\ldots,\rho_D^2)$.
    \STATE $S \LET \sum_{j=1}^r \Kernel_j$, $M \LET \alpha \alpha\transpose$
\STATE $\delta(d') \LET \rho_{d'}^{-2} \ip{ M, S^{\odot d'}}, \quad d'\in\{0,\ldots,D\}$
\STATE Sample $d$ from $\delta(\cdot)/ \sum_{d'=0}^D \delta(d')$
\FOR{$i=1$ to $d$}
    \STATE $\pi(j) \LET
     \frac{\tr( M \, S^{\odot(d-i)} \odot \Kernel_j ) }
            {\tr( M \, S^{\odot(d-i+1)} ) }, \quad j\in\{1,\ldots,r\}$
    \STATE Sample $z_i$ from $\pi(\cdot)$
    \STATE $M \LET M \odot \Kernel_{z_i}$
\ENDFOR
\STATE {\bf return} $(z_1,\ldots,z_d)$ 
\end{algorithmic}
\end{algorithm}

\subsection{Correctness of the sampling procedure}
\label{sec:polysampling}
In this section we prove the correctness of 
Algorithm~\ref{alg:poly_kernel_sampling}. 


As said earlier, we assume that $\rho_{\mi}$ depends only on $d$, the order of interactions in $\kernel_{\mi}$
and, by abusing notation, we will write $\rho_d$  to emphasize this.
Let us now consider 
how one can sample from $q_{k-1,\cdot}$. 
The implementation relies on the fact that $(\sum_{j=1}^r \kernel_j)^d=\sum_{\mi \in \I} \kernel_\mi$.

Remember that we denoted the kernel matrix underlying some kernel $k$ by $\Kernel_k$, and recall that $\Kernel_k$ is an $n \times n$ matrix. 
For brevity, in the rest of this section for $\kernel = \kernel_{\mi}$ we will write $\Kernel_{\mi}$ instead of $\Kernel_{\kernel_\mi}$.
Define $M_{k-1} = \alpha^*(\theta_{k-1}) \alpha^*(\theta_{k-1})^\top$.
Thanks to~\eqref{eq:diff1} and the rotation property of trace, we have
\begin{align}\label{eq:gradpoly}
g_{k,\mi} = - \rho_{d}^{-2} \tr (M_{k-1} \, \Kernel_\mi)\,.
\end{align}
The plan to sample from $q_{k-1,\cdot} = |g_{k,\cdot}| /\sum_{\mi \in \I} |g_{k,\mi}|$ is as follows:
We first draw the order of interactions, $0\le \hd \le D$.
Given $\hd=d$, we restrict the draw of the random multi-index  $\MI$ to the set $\{\mi \in \I\}$.
A multi-index will be sampled in a $\hd$-step process: in each step we will randomly choose an index from the indices of base kernels according to the following distributions.
Let $S = \Kernel_1+\ldots+\Kernel_r$, let
\[
\Prob{\hd= d|\F_{k-1}} =  \frac{\rho_d^{-2} \tr( M_{k-1} S^{\odot d})}
 									{ \sum_{d'=0}^D \rho_{d'}^{-2} \tr( M_{k-1} S^{\odot d'})
									}
\]
and, with a slight abuse of notation, for any $1 \le i \le d$ define
\begin{eqnarray*}
\lefteqn{ \Prob{R_i = r_i |\F_{k-1}, \hd = d, R_{1:i-1} = r_{1:i-1}} } \\
 					&=& \frac{ \tr \left( M_{k-1}\, \odot \left(\odot_{j=1}^i \Kernel_{r_j}\right) \odot S^{\odot (d-i)}\right)}
 												    { \sum_{r_i'=1}^r \tr \left( M_{k-1}\,\odot \left(\odot_{j=1}^{i-1}\Kernel_{r_j}\right) \odot \Kernel_{r_i'}\odot S^{\odot (d-i)}\right)}
\end{eqnarray*}
where we used the sequence notation (namely, $s_{1:p}$ denotes the sequence $(s_1,\ldots,s_p)$).
We have, by the linearity of trace and the definition of $S$ that
\begin{eqnarray*}
\lefteqn{\sum_{r_i'=1}^r \tr \left( M_{k-1}\,\odot \left(\odot_{j=1}^{i-1}\Kernel_{r_j}\right) \odot \Kernel_{r_i'}\odot S^{\odot (d-i)}\right)} \\
&=& \tr \left( M_{k-1}\,\odot\left(\odot_{j=1}^{i-1}\Kernel_{r_j}\right) \odot S^{\odot (d-i+1)}\right)
\end{eqnarray*}
Thus, by telescoping,
\begin{eqnarray*}
 \lefteqn{\Prob{ \hd = d, R_{1:d} = r_{1:d}   | \F_{k-1} } }\\
&=&
\frac{ \rho_d^{-2} \tr ( M_{k-1}\, \Kernel_{r_1}\odot\ldots \odot \Kernel_{r_{d-1}} \odot \Kernel_{r_d})}
 									{ \sum_{d'=0}^D \rho_{d'}^{-2} \tr( M_{k-1} S^{\odot d'})
									}.
\end{eqnarray*}
as desired.
An optimized implementation of drawing these random variables is shown as Algorithm~\ref{alg:poly_kernel_sampling}. 
 The algorithm is written to return the multi-index $\MI$.

\section{Experiments}
\label{sec:experiments}
In this section we apply our method to the problem of multiple kernel learning in regression with the squared loss:
$L(w) = \frac12 \sum_{t=1}^n (f_w(x_t) -y_t)^2$, where $(x_t,y_t)\in \real^r\times \real$ are the input-output pairs in the data. 
In these experiments our aim is to learn polynomial kernels (cf. Section~\ref{sec:poly}).

We compare our method against several kernel learning algorithms from the literature on synthetic and real data. 
In all experiments we report mean squared error over test sets. A constant feature is added to act as offset, and the inputs and output are normalized to have zero mean and unit variance. Each experiment is performed with $10$ runs in which we randomly choose training, validation, and test sets. The results are averaged over these runs. 

\subsection{Convergence speed\todoar{is this title good? if not, change it.}}

In this experiment we examine the speed of convergence of our method and compare it against one of the fastest standard multiple kernel learning algorithms, that is, the $p$-norm multiple kernel learning algorithm of \citet{kloft2011lp} with $p=2$,\footnote{Note that $p=2$ in \citet{kloft2011lp} notation corresponds to $p=4/3$ or $\nu=2$ in our notation, which gives the same objective function that we minimize with Algorithm~\ref{alg:grad}.} and the uniform coordinate descent algorithm that updates one coordinate per iteration uniformly at random \citep{Nes10,Nes12, ShTe11,RichTa11}. We aim to learn polynomial kernels of up to degree $3$ with all algorithms. Our method uses Algorithm \ref{alg:poly_kernel_sampling} for sampling with $D=3$. The set of provided base kernels is the linear kernels built from input variables, that is, $\kernel_{(i)}(x,x')=x_{(i)} x'_{(i)}$, where $x_{(i)}$ denotes the $i^{\rm th}$ input variable. For the other two algorithms the kernel set consists of product kernels from monomial terms for $D\in\{0,1,2,3\}$ built from $r$ base kernels, where $r$ is the number of input variables. The number of distinct product kernels is $\binom{r+D}{D}$. In this experiment for all algorithms we use ridge regression with its regularization parameter set to $10^{-5}$. Experiments with other values of the regularization parameter achieved similar results.
 
We compare these methods in four datasets from the UCI machine learning repository \citep{frank2010uci} and the Delve datasets\footnote{See, \url{www.cs.toronto.edu/~delve/data/datasets.html}}. The specifications of these datasets are shown in Table~\ref{tab:dataset_specs}. 
\begin{table*}[htbp]
\caption{Specifications of datasets used in experiments.}
\label{tab:dataset_specs}
\begin{center}
\begin{small}
\begin{tabular}{|l|l|l|l|l|l|}
  \hline
	\textbf{Dataset} & \textbf{\# of variables} & \textbf{Training size} & \textbf{Validation size} & \textbf{Test size} \\ \hline
german & $20$ & $350$ & $150$ & $500$ \\ \hline
ionosphere & $34$ & $140$ & $36$ & $175$ \\ \hline
ringnorm & $20$ & $500$ & $1000$ & $2000$ \\ \hline
sonar & $60$ & $83$ & $21$ & $104$ \\ \hline
splice & $60$ & $500$ & $1000$ & $1491$ \\ \hline
waveform & $21$ & $500$ & $1000$ & $2000$ \\ \hline
\end{tabular}
\end{small}
\end{center}
\end{table*}
We run all algorithms for a fixed amount of time and measure the value of the objective function~\eqref{eq:penrisk}, that is, the sum of the empirical loss and the regularization term. Figure~\ref{fig:convergence_test} shows the performance of these algorithms. In this figure \Stoch represents our algorithms, \Kloft represents the algorithm of \citet{kloft2011lp}, and \UCD represents the uniform coordinate descent algorithm. 
\begin{figure*}[tb]
\begin{center}
\centerline{\includegraphics[width=1.2\linewidth]{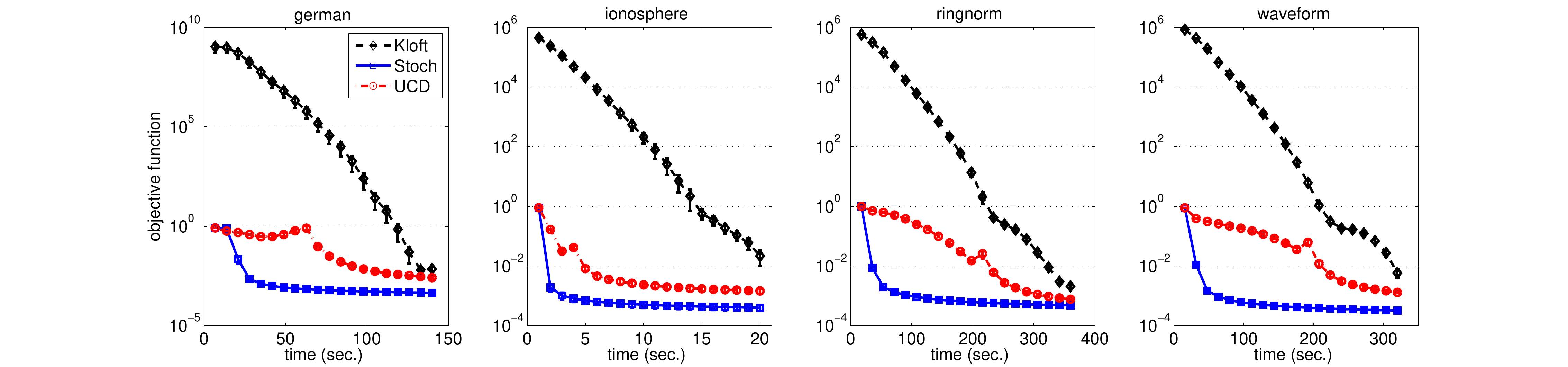}}
\caption{Convergence comparison of our method and other algorithms.}
\label{fig:convergence_test}
\end{center}
\vskip -0.3in
\end{figure*}
The results show that our method consistently outperforms the other algorithms in convergence speed. Note that our stochastic method updates one kernel coefficient per iteration, while \Kloft updates $\binom{r+D}{D}$ kernel coefficients per iteration. The difference between the two methods is analogous to the difference between stochastic gradient vs. full gradient algorithms. While \UCD also updates one kernel coefficient per iteration its naive method of selecting coordinates results in a slower overall convergence compared to our algorithm. In the next section we compare our algorithm against several representative methods from the MKL literature.

\subsection{Synthetic data}

In this experiment we examine the effect of the size of the kernel space on prediction accuracy and training time of MKL algorithms. We generated data for a regression problem. Let $r$ denote the number of dimensions of the input space. The inputs are chosen uniformly at random from $[-1,1]^r$. The output of each instance is the uniform combination of $10$ monomial terms of degree $3$ or less. These terms are chosen uniformly at random among all possible terms. The outputs are noise free. We generated data for $r\in\{5,10,20,\ldots,100\}$, with $500$ training and $1000$ test points. The regularization parameter of the ridge regression algorithm was tuned from $\{10^{-8},\ldots,10^2\}$ using a separate validation set with $1000$ data points.

We compare our method (\Stoch) against the algorithm of \citet{kloft2011lp} (\Kloft), the nonlinear kernel learning method of \citet{cortes2009learning} (\Cortes), and the hierarchical kernel learning algorithm of \citet{bach2008exploring} (\Bach).%
\footnote{While several fast MKL algorithms are available
in the literature, such as those of \citet{sonnenburg2006large,rakotomamonjy2008simplemkl,Xu10-MKLLasso,orabona11-sparseMKL,kloft2011lp}, 
a comparison of the reported experimental results shows that from among these algorithms the method of \citet{kloft2011lp} has the best performance overall.
Hence, we decided to compare against only this algorithm. Also note that the memory and computational cost of all these methods still scale linearly with the number of kernels, making them unsuitable for the case we are most interested in. Furthermore, to keep the focus of the paper we compare our algorithm to methods with sound theoretical guarantees. As such, it remains for future work to compare with other methods, such as the infinite kernel learning of \citet{gehler2008infinite}, which lack such guarantees but exhibit promising performance in practice.
}
 The set of base kernels consists of $r$ linear kernels built from the input variables. Recall that the method of \citet{cortes2009learning} only considers kernels of the form $\kernel_{\theta} = (\sum_{i=1}^r \theta_i \kernel_i)^D$, where $D$ is a predetermined integer that specifies the degree of nonlinear kernel. Note that adding a constant feature is equivalent to adding polynomial kernels of degree less than $D$ to the combination too. We provide all possible product kernels of degree $0$ to $D$ to the kernel learning method of \citet{kloft2011lp}. For our method and the method of \citet{bach2008exploring} we set the maximum kernel degree to $D=3$.

\begin{figure*}[!tb]
\begin{center}
\centerline{\includegraphics[width=1\linewidth]{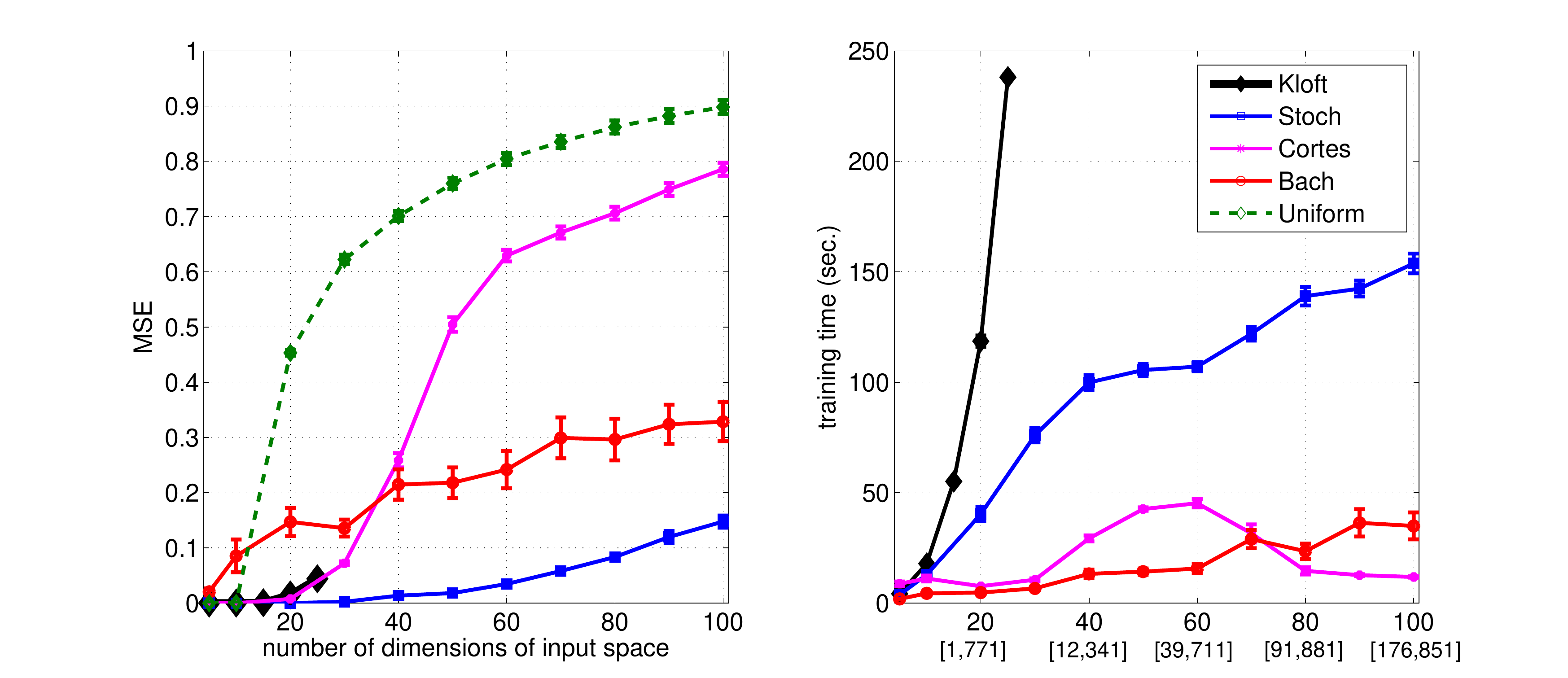}}
\caption{Comparison of kernel learning methods in terms of test error (left) and training time (right).}
\label{fig:scalability_test}
\end{center}
\vskip -0.3in
\end{figure*}

The results are shown in Figure~\ref{fig:scalability_test}, the mean squared errors are on the left plot, while the training times are on the right plot. In the training-time plot the numbers inside brackets indicate the total number of distinct product kernels for each value of $r$. This is the number of kernels fed to the \Kloft algorithm. Since this method deals with a large number of kernels, it was possible to precompute and keep the kernels in memory ($8$GB) for $r \leq 25$. Therefore, we ran this algorithm for $r \leq 25$. 
For $r > 25$, we could use on-the-fly implementation of this algorithm, however that further increases the training time. Note that the computational cost of this method depends linearly on the number of kernels, which in this experiment, is cubic in the number of input variables since $D=3$. While the standard MKL algorithms, such as \Kloft, cannot handle such large kernel spaces, in terms of time and space complexity, the other three algorithms can efficiently learn kernel combinations. However their predictive accuracies are quite different. Note that the performance of the method of \citet{cortes2009learning} starts to degrade as $r$ increases. This is due to the restricted family of kernels that this method considers. The method of \citet{bach2008exploring}, which is well-suited to learn sparse combination of product kernels, performs better than \citet{cortes2009learning} for higher input dimensions. Among all methods, our method performs best in predictive accuracy while its computational cost is close to that of the other two competitors.

\subsection{Real data}
\label{sec:real_data_exp}

In this experiment we aim to compare several MKL methods in real datasets. We compare our new algorithm (\Stoch), the algorithm of \citet{bach2008exploring} (\Bach), and the algorithm of \citet{cortes2009learning} (\Cortes). For each algorithm we consider learning polynomial kernels of degree $2$ and $3$. We also include uniform combination of product kernels of degree $D$, i.e. $\kernel_D=\left( \sum_{i=1}^r \kernel_i \right)^D$, for $D\in\{1,2,3\}$ (\Uniform). To find out if considering higher-order interaction of input variables results in improved performance we also included a MKL algorithm to which we only feed linear kernels ($D=1$). We use the MKL algorithm of \citet{kloft2011lp} with $p\in\{1,2\}$ (\Kloft).

We compare these methods on six datasets from the UCI machine learning repository and Delve datasets. In these datasets the number of dimensions of the input space is $20$ and above.
The specifications of these datasets are shown in Table~\ref{tab:dataset_specs}. The regularization parameter is selected from the set $\{10^{-4},\ldots,10^3\}$ for all methods using a validation set. The results are shown in Figure~\ref{fig:mldata}. 
 
Overall, we observe that methods that consider non-linear variable interactions (\Stoch, \Bach, and \Cortes) perform better than linear methods (\Kloft). Among non-linear methods, \Cortes performs worse than the other two. We believe that this is due to the restricted kernel space considered by this method. The performance of \Stoch and \Bach methods is similar overall.

We observe that our method overfits when it considers kernels of degree $3$. However, one can easily prevent overfitting by assigning larger $\rho$ values to higher-degree kernels such that the stochastic algorithm selects lower-degree kernels more often. For this purpose, we repeat this experiment for $D = 3$ with a modified set of $\rho$ values, where we use $\rho_d^2 = 1$ for kernels of degree $2$ or less and $\rho_d^2 = 4$ for kernels of degree $3$. With the new $\rho$ coefficients we observe an improvement in algorithm's performance. \todoc{To some extent and ``can easily dealt with'' contradict. How about using $\rho_i$ that would be respected by a sampling density!??} See \Stoch ($D=3$, prior) error values in Figure~\ref{fig:mldata}. 
\begin{figure*}[tb]
\begin{center}
\centerline{\includegraphics[width=1\linewidth]{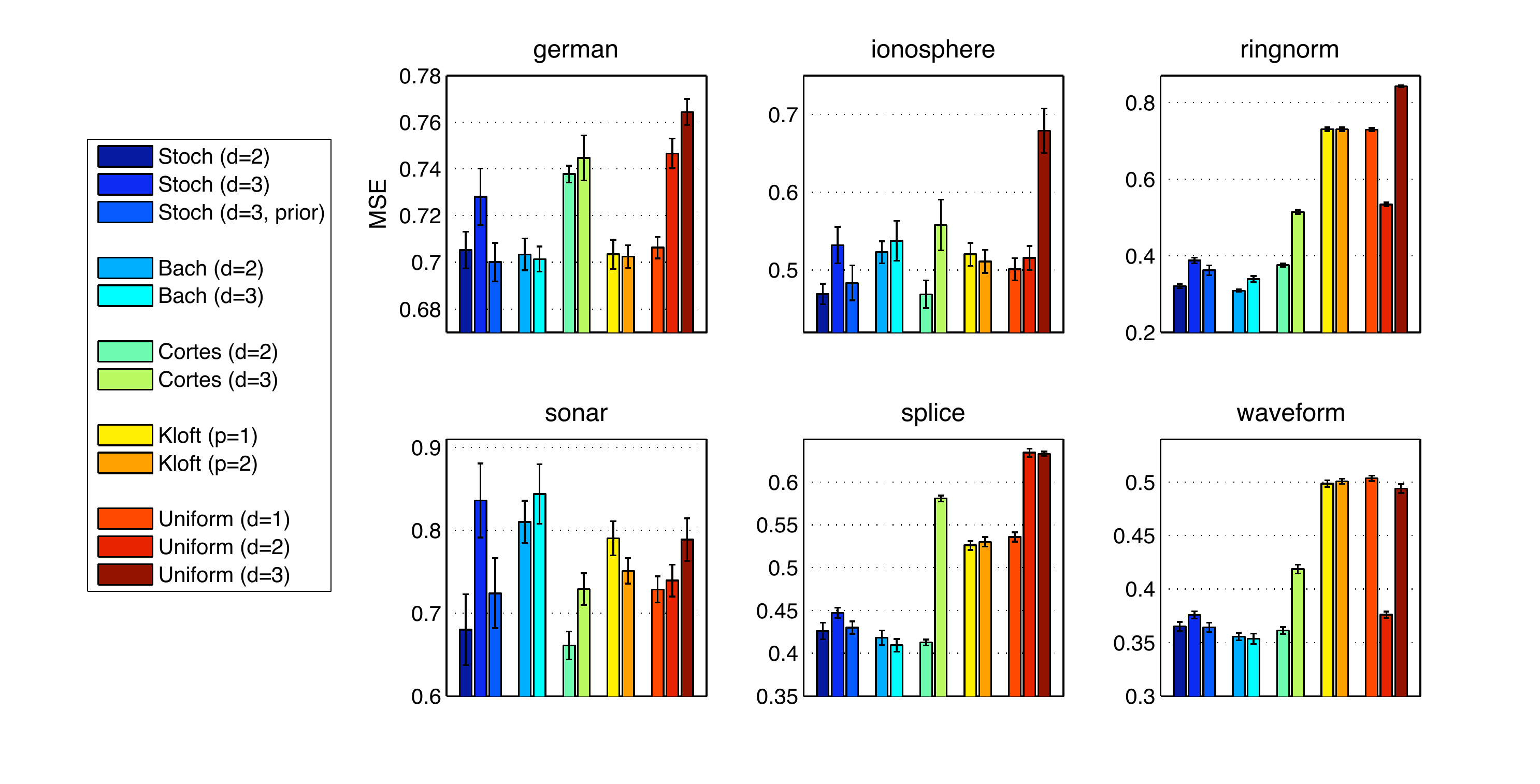}}
\vspace{-0.2cm}
\caption{Prediction error of different methods in the real data experiment}
\label{fig:mldata}
\end{center}
\vskip -0.3in
\end{figure*}

\section{Conclusion}
\label{sec:conclusion}
We introduced a new method for learning a predictor by combining exponentially many linear predictors using a randomized mirror descent algorithm.  We derived finite-time performance bounds that show that the method efficiently optimizes our proposed criterion.
Our proposed method is a variant of a randomized stochastic coordinate descent algorithm, where the main trick is the careful construction of an unbiased randomized estimate of the gradient vector that keeps the variance 
 of the method under control,
 and can be computed efficiently when the base kernels have a certain special combinatorial structure. 
The efficiency of  our method was demonstrated for the practically important problem of learning polynomial kernels on a variety of synthetic and real datasets comparing to a representative set of algorithms from the literature. 
For this case, our method is able to compute an optimal solution in polynomial time as a function of the \emph{logarithm} of the number of base kernels.
To our knowledge, ours is the first method for learning kernel combinations that achieve such an exponential reduction in complexity while satisfying strong performance guarantees, thus opening up the way to apply it to extremely large number of kernels.
Furthermore, we believe that our method is applicable beyond the case studied in detail in our paper.
For example, the method seems extendible to the case when infinitely many kernels are combined, such as the case of learning a combination of Gaussian kernels.
However, the investigation of this important problem remains subject to future work.

\subsubsection*{Acknowledgements}
This work was supported by Alberta Innovates Technology Futures and NSERC.





\appendix

\section{Proofs}
\label{sec:proofs}

In this section we present the proofs of Theorem~\ref{thm:proximal} and Proposition~\ref{prop:wstar}. The proof of Theorem~\ref{thm:proximal} is based on the standard proof of the convergence rate of the proximal point algorithm, see, for example, \citep{beck2003mirror}, or the proof of Proposition~2.2 of \citet{nemirovski2009robust}, which carry over the same argument to solve very similar but less general problems. We also provide some improvements and simplifications at the end. 
Before giving the actual proof, we need the following standard lemma:

\begin{restatable}[Lemma 2.1 of \citealt{nemirovski2009robust}]{lemma}{egylemma}
\label{lem:nemirovski}
Assume that $\Psi$ is $\alpha$-strongly convex  with respect to some norm $\|\cdot\|$ (i.e., \eqref{eq:psi-strong} holds).
Let $\t_1 \in K \cap A^\circ$, $\t \in K\cap A$, and $g \in \R^d$. Define
$\t_2 = \argmin_{\t' \in K\cap A} \left\{ \langle g, \t' \rangle + D_{\Psi}(\t',\t_1)\right\}$. Then
\[
\ip{g, \t_1 - \t}  \le D_\Psi(\t,\t_1) - D_\Psi(\t,\t_2) + \frac{\| g \|_*^2}{2\alpha}.
\]
\end{restatable} 
\renewcommand{\div}{D_{\Psi}}
We provide an alternate proof that is based on the so-called $3$-DIV lemma. The $3$-DIV lemma \citep[e.g., Lemma 11.1,][]{CBLu06:book} allows one to express the sum of the divergences between the vectors $u,v$ and $v,w$ in terms of the divergence between $u$ and $w$ and an additional ``error term'', where $u\in A$, $v,w\in A^\circ$:
\[
\div(u,v) + \div(v,w) = \div(u,w) + \ip{ \nabla \psi(w) - \nabla \psi(v),u-v}\,.
\]
\begin{proof}
Note that $\t_2\in A^\circ$ due to behavior of $\Psi$ at the boundary of $A$. Thus, $\Psi$ is differentiable at $\t_2$ and 
\begin{align}\label{eq:divdiff}
\nabla_1 \div(\t_2,\t_1) = \nabla \psi(\t_2) - \nabla \psi(\t_1)\,,
\end{align}
where $\nabla_1$ denotes differentiation of $\div$ w.r.t. its first variable.
Let $f(\t') = \ip{g,\t'} + \div(\t',\t_1)$. By the optimality property of $\t_2$ and since $\t\in K\cap A$,  we  have
\[
\ip{ \nabla f(\t_2), \t_2 - \t } \le 0\,.
\]
Plugging in the definition of $f$ together with the identity~\eqref{eq:divdiff} gives 
\begin{align}\label{eq:opt2}
\ip{ g + \nabla \psi(\t_2) - \nabla \psi(\t_1), \t_2 - \t } \le 0\,.
\end{align}
Now, by the $3$-DIV Lemma, 
\begin{align*}
\div(\t,\t_2) + \div(\t_2,\t_1) 
 & = \div(\t,\t_1) + \ip{ \nabla \Psi(\t_1) - \nabla \Psi(\t_2), \t - \t_2 } \\
 & = \div(\t,\t_1) + \ip{ g + \nabla \Psi(\t_2) - \nabla \Psi(\t_1),\t_2-\t} + \ip{g, \t-\t_2}.
\end{align*}
Hence, by reordering and using the inequality~\eqref{eq:opt2} we get
\begin{align*}
\div(\t,\t_2) - \div(\t,\t_1) &\le \ip{g,\t-\t_2} - \div(\t_2,\t_1) \\
& = \ip{g,\t_1-\t_2} -  \div(\t_2,\t_1) + \ip{g,\t-\t_1} \\
& \le \frac{\|g\|_*^2}{2\alpha} + \ip{g,\t-\t_1}\,,
\end{align*}
where in the last line we used Young's inequality\footnote{Young's inequality states that for any $x,y$ vectors and $\alpha>0$, $\ip{x,y} \le \|x\|_* \|y\| \le \frac12 \left(\frac{\|x\|_*^2}{\alpha}+\alpha \|y\|^2\right)$.} and that due to the strong convexity of $\Psi$, $\div(\t_2,\t_1) \ge \frac{\alpha}{2} \| \t_2 - \t_1 \|^2$.
\end{proof}
\newcommand{\oeta}{\overline{\eta}^{(T)}}
\theoremPerfBound*
\newcommand{\gradk}{g_k}
\begin{proof}
Introduce the average learning rates $\oeta_{k} = \eta_k/{\sum_{k=1}^{T} \eta_{k-1}}$, $k=1,\ldots,T$, the averaged parameter estimates
\[
\bar{\t}^{(T-1)}=\sum_{k=1}^{T} \oeta_{k-1} \t^{(k-1)}
\] 
and choose some $\t^*\in K\cap A$. To prove the first part of the theorem, it suffices to show that the bound holds for $J(\bar{\t}^{(T-1)}) - J(\t^*)$. Define $\gradk=\nabla J \left (\t^{(k-1)}\right)$.
By the convexity of $J(\t)$, we have
\begin{eqnarray}
J\left(\bar{\t}^{(T-1)}\right) - J(\t^*)
&\le& \sum_{k=1}^{T} \oeta_{k-1}  \left(J\left(\t^{(k-1)}\right)-J(\t^*)\right) \nonumber \\
&\le& \sum_{k=1}^T \oeta_{k-1}  \ip{\gradk, \t^{(k-1)}-\t^* } \nonumber \\
&=& \sum_{k=1}^T \oeta_{k-1}  \ip{\hg_k, \t^{(k-1)}-\t^*} +
\sum_{k=1}^T \oeta_{k-1}  \ip{ \gradk - \hg_k, \t^{(k-1)}-\t^*}
\label{eq:thm-J}
\end{eqnarray}
Notice that the first term on the right hand side above is the sum of linearized losses appearing in the standard analysis of the proximal point algorithm with loss functions $\hg_k$ and learning rates $\oeta_{k-1}$, and the second sum contains the term that depends on how well $\hg_k$ estimates the gradient 
$\gradk$. Thus, in this way, it is separated how the proximal point algorithm and the gradient estimate effect the convergence rate of the algorithm.
The first sum can be bounded by invoking the standard bound for the proximal point algorithm (we will give the very short proof for completeness, based on  Lemma~\ref{lem:nemirovski}), while the second sum can be analyzed by noticing that, by assumption \eqref{eq:gradest}, its elements form an $\{\F_{k}\}$-adapted martingale-difference sequence.

To bound the first sum,
 first note that the conditions of Lemma~\ref{lem:nemirovski} are satisfied for $\t_1=\t^{(k-1)},\t=\t^*, g=\oeta_{k-1} \hg_k$,
since $\t_1\in K\cap A^\circ$ (as mentioned beforehand, this follows from the behavior of $\Psi$ at the boundary of $A$). Further, note that due to the so-called projection lemma (i.e., the $\div$-projection of the unconstrained optimizer is the same as the optimizer of the constrained optimization problem),\todo{A reference would be nice} %
we can conclude that $\t^{(k)}=\t_2$, where $\t_2$ is defined in Lemma~\ref{lem:nemirovski}.
Thus, Lemma~\ref{lem:nemirovski} gives
\[
\eta_{k-1}\ip{ \hg_k,\t^{(k-1)}-\t^*}
\le \div(\t^*,\t^{(k-1)}) - \div(\t^*,\t^{(k}) + \frac{\eta_{k-1}^2 \|\hg_k\|_*^2}{2\alpha}.
\]
Summing the above inequality for $k=1,\ldots,T$, the divergence terms cancel each other, yielding
\begin{equation}
\label{eq:thm-J1}
\sum_{k=1}^T \oeta_{k-1}  \ip{\hg_k, \t^{(k-1)}-\t^*} \le \frac{1}{\sum_{k=1}^{T} \eta_{k-1}}\left(\div(\t^*,\t^{(0)}) - \div(\t^*,\t^{(T)}) + \frac{1}{2\alpha} \sum_{k=1}^T \eta_{k-1}^2 \|\hg_k\|_*^2 \right)~.
\end{equation}

Let us now turn to the second sum. 
We start with developing a bound on the expected regret.
For any $1\le k \le T$, by construction  $\oeta_{k-1}$ and $\t^{(k-1)}$ are $\F_{k-1}$-measurable. This, together with~\eqref{eq:gradest} gives
\begin{equation}
\label{eq:zmartingale}
\EEpcond{\oeta_{k-1}\ip{\gradk -\hg_k,\t^*-\t^{(k-1)} }}{\F_{k-1}} = 
\oeta_{k-1}\ip{\gradk- \EEpcond{\hg_k}{\F_{k-1}} , \t^*-\t^{(k-1)} } = 0~.
\end{equation}
Combining this result with \eqref{eq:thm-J} and \eqref{eq:thm-J1} yields
\begin{eqnarray}
\EEp{J\left(\bar{\t}^{(T)}\right) - J(\t^*)} 
&\le& 
\frac{1}{\sum_{k=1}^{T} \eta_{k-1}}\left(\div(\t^*,\t^{(0)}) - \div(\t^*,\t^{(T)}) + \frac{1}{2\alpha} \sum_{k=1}^T \eta_{k-1}^2 \EEp{\EEpcond{\|\hg_k\|_*^2}{\F_{k-1}}} \right) \nonumber \\
&\le&\frac{ \delta + \frac{1}{2\alpha} \sum_{k=1}^T \eta_{k-1}^2 B}{\sum_{k=1}^{T} \eta_{k-1}}~,
\label{eq:thm-J2}
\end{eqnarray}
where we used the tower rule to bring in the bound \eqref{eq:gbound}, the nonnegativity of Bregman divergences, and
$D_\Psi(\t,\t^{(0)}) \le \Psi(\t)-\Psi(\t^{(0)})$; the latter holds as $\ip{ \nabla \Psi(\t^{(0)}), \t-\t^{(0)} } \ge 0$ since $\t^{(0)}$ minimizes $\Psi$ on $K$.
 \todoc{Note that $\EEp{\|\hg_k\|_*^2}$ could be smaller than $\sup_{k} \EEpcond{\|\hg_k\|_*^2}{\F_{k-1} }$. Thus, the result on the expected rate of convergence could be strengthened. Maybe add a remark after the theorem!}
Substituting $\eta_{k-1}=\eta=\sqrt{\frac{2\alpha\delta}{B T}}, k=1,\ldots,T$ finishes the proof of \eqref{eq:thm-expectation}.
\todoc{We should note somewhere that a time-varying learning rate can also be used easily.}

To prove the high probability result~\eqref{eq:thm-highprob}, 
 notice that thanks to~\eqref{eq:gradest}  
  $\left\{ \eta_{k-1} \ip{\gradk -\hg_k,\t^*-\t^{(k-1)}} \right\}$ is an $\{\F_{k}\}$-adapted martingale-difference sequence (cf.~\eqref{eq:zmartingale}).
By the strong convexity of $\Psi$ we have 
\[
\frac{\alpha}{2} \|\t^{(k-1)}-\t^*\|^2 \le \Psi(\t^{(k-1)})-\Psi(\t^*) \le \delta.
\]
Furthermore, conditions \eqref{eq:gradest} and \eqref{eq:hgkboundedvar} imply that
$\left\| \gradk\right\|_*^2 \le B'$ a.s., and so by  \eqref{eq:hgkboundedvar} we have
 $\left\|\gradk -\hg_k\right\|_* \le  2\sqrt{B'}$ a.s. Then by H\"older's inequality
\[
\left| \ip{\gradk -\hg_k,\t^*-\t^{(k-1)}} \right|
\le \left\|\gradk -\hg_k\right\|_* \|\t^*-\t^{(k-1)}\| \le 2\sqrt{\frac{2B'\delta}{\alpha}}.
\]
Thus, by the Hoeffding-Azuma inequality \citep[see, e.g., Lemma A.7,][]{CBLu06:book}, for any $0<\epsilon<1$ we have, with probability at least $1-\epsilon$,
\begin{equation}
\label{eq:z-hpbound}
\sum_{k=1}^T \oeta_{k-1} \ip{\gradk -\hg_k ,\t^*-\t^{(k-1)}} \le 
\frac{4}{\sum_{k=1}^T \eta_{k-1}}\sqrt{\frac{B'\delta}{\alpha}\left(\sum_{k=1}^T \eta_{k-1}^2\right)\ln\frac{1}{\epsilon}}~.
\end{equation}
Combining \eqref{eq:thm-J1} with \eqref{eq:hgkboundedvar} implies an almost sure upper bound on the first sum
on the right hand side of \eqref{eq:thm-J} as in \eqref{eq:thm-J2} with $B'$ in place of $B$.
\todoc{Missing some terms $\sum_{k=1}^T\eta_{k-1}^2\approx \sum_{k=1}^T 1/k \approx \ln(T)$?? Note that for $\eta_k=1/k$, we are in trouble: $\sum_{k} 1/k^2$ is finite, but $\sum_{k=1}^T (1/k) \approx \ln (T)$, so a term like $1/\ln(T)$ shows up!? What am I missing? The whole proof seems to be flawn for the strongly convex case because of this. A different argument will likely be needed.}
 This, together with 
\eqref{eq:z-hpbound} proves the required high probability bound \eqref{eq:thm-highprob} when substituting
$\eta_{k-1}=\eta'=\sqrt{\frac{2\alpha\delta}{B'T}}$.

\end{proof}

\propwstar*
\begin{proof}
By introducing the variables $\tau = (\tau_t)_{1\le t\le n}\in \real^n$  and using the definition of $L$
we can write the optimization problem~\eqref{eq:jmin} as the constrained optimization problem
\begin{align}\label{eq:primal}
\begin{split}
\minimize_{w\in \W,\tau\in \real^n} \,\, & \frac1n \sum_{t=1}^n \ell_t(\tau_t) + \frac12 \sum_{i\in \I} \frac{\rho_i^2 \|w_i\|_2^2}{\theta_i}
\qquad \st  \tau_t = \sum_{i\in \I} \ip{ w_i, \phi_i(x_t) },
\end{split}
\end{align}
In what follows, we call this problem the primal problem.
The Lagrangian of this problem is
\[
\L(w,\tau,\alpha) \eqdef  \frac1n \sum_{t=1}^n \ell_t(\tau_t) +
 \frac12 \sum_{i\in \I} \frac{\rho_i^2 \|w_i\|_2^2}{\theta_i} + \sum_{t=1}^n \alpha_t \left\{\tau_t - \sum_{i\in \I} \ip{ w_i, \phi_i(x_t) }\right\}\,,
 \]
 where $\alpha = (\alpha_t)_{1\le t\le n} \in \real^n$ is the vector of Lagrange multipliers (or dual variables) associated with the $n$ equality constraints. The Lagrange dual function, $g(\alpha) \eqdef \inf_{w,\tau} \L(w,\tau,\alpha)$, can be readily seen to satisfy
 \[
g(\alpha) =   -\left( \frac{1}{2} \alpha\transpose \Kernel_\t \alpha + \frac{1}{n} \sum_{t=1}^n \ell_t^*(-n\alpha_t)\right)\,.
 \]\todoar{somebody please do the derivations here. I changed the above formulation a little to take $\frac{1}{n}$ factor into account}
Now, since the objective function of the primal problem is convex and the primal problem involves only affine equality constraints and the primal problem is clearly feasible, 
 by Slater's condition \citep[p.226,][]{boyd2004convex}, if $\alpha^*(\t)$ is the maximizer of $g(\alpha)$ then 
 \begin{align*}
 w^*(\t) &= \argmin_{w\in \W} \inf_{\tau\in \real^n} \L(w,\tau,\alpha^*(\t))\\
& =
\argmin_{w\in \W}  
 \sum_{i\in \I} \left\{ \frac{\rho_i^2 \|w_i\|_2^2}{2\theta_i} - \sum_{t=1}^n \alpha_t \ip{ w_i, \phi_i(x_t) } \right\}\,.
\end{align*} 
The minimum of the last expression is readily seen to be equal to the expression given in~\eqref{eq:optimal_w},
thus finishing the proof.
\end{proof}

\section{Calculating the derivative of $J(\t)$}
\label{sec:derivative}
In this section we show that under mild conditions the derivative of $J$ exist and we also give explicit forms. 
These derivations are quite standard and a similar argument can be found in the paper by (e.g.) \citet{rakotomamonjy2008simplemkl} specialized to the case when  $\ell_t$ is the hinge loss.

As it is well-known, thanks to the implicit function theorem \citep[e.g.,][Theorem~7.5.6]{BrPa70}, 
provided that $J=J(w,\t)$ is such that $\frac{\partial^2}{\partial \t \partial w} J(w,\t)$ and $\frac{\partial}{\partial w} J(w,\t)$ are continuous,
the gradient of $J(\t)$ can be computed by evaluating the partial derivative $\frac{\partial}{\partial \t} J(w,\t)$ of $J(w,\t)$ with respect to $\t$ at $(w^*(\t),\t))$, that is,
$\partial_\t J(\t) = \frac{\partial}{\partial \t} \left.J(w,\t)\right|_{w=w^*(\t)}$.
\todoc{We should rather do what A Rakotomamonjy et al. do in  their SimpleMKL paper. Bonnans and Shapiro, 1998 is the relevant result.}
Note that the derivative is well-defined only if $\theta>0$, that is, when no coordinates of $\theta$ is zero, in which case
\begin{equation}
\label{eq:diff}
\frac{\partial}{\partial \t} J(w^*(\t),\t) = -\left(\frac{\rho_i^2 \|w^*_i(\t)\|_2^2}{\t_i^2}\right)_{i\in\I}.
\end{equation}
If $\theta_i=0$ for some $i\in\I$, we define the derivative in a continuous manner as
\begin{equation}
\label{eq:diff0}
\frac{\partial}{\partial \t} J(\t) =  \lim_{\substack{\t'\to\t \\ \t'\in\D, \t'>0}} \frac{\partial}{\partial \t} J(\t')
\end{equation}
assuming that the limit exists. 
From \eqref{eq:optimal_w} we get, for any $i\in\I$,
$ 
\|w^*_i(\t)\|_2^2=\frac{\t_i^2}{\rho_i^4} \alpha^*(\t)\transpose \Kernel_i \alpha^*(\t)
$. 
Combining with \eqref{eq:diff} we obtain
\begin{equation*}
\frac{\partial}{\partial \t} J(w^*(\t),\t) = -\left(\frac{\alpha^*(\t)\transpose \Kernel_i \alpha^*(\t)}{ \rho_i^2}\right)_{i\in\I}\,.
\end{equation*}
Now, by \eqref{eq:diff0} and the implicit function theorem, $\alpha^*(\t)$ is a continuous function of $\t$ provided that the functions $\ell_t^*$ ($1\le t\le n$) are twice continuously differentiable. 
This shows that under the conditions listed so far, the limit in \eqref{eq:diff0} exists.
In the application we shall be concerned with, these conditions can be readily verified.



{
\addtolength{\parskip}{-0.2cm}
\bibliography{references}
\bibliographystyle{apalike}
}

\end{document}